\pdfoutput=1
\documentclass{article}

\usepackage[nonatbib,preprint]{neurips_2023}

\usepackage{color}

\usepackage[utf8]{inputenc} 
\usepackage[T1]{fontenc}    
\usepackage{hyperref}       
\usepackage{url}            
\usepackage{booktabs}       
\usepackage{amsfonts}       
\usepackage{nicefrac}       
\usepackage{microtype}      
\usepackage{xcolor}         
\usepackage{amsmath}
\usepackage{enumitem}
\usepackage{amsthm}
\usepackage{graphicx}
\usepackage{multicol}
\usepackage{multirow}
\usepackage{algorithm}
\usepackage[noend]{algpseudocode}
\newtheorem{theorem}{Theorem}

\def\name{SHAP@k}

\DeclareMathOperator*{\argmax}{arg\,max}
\DeclareMathOperator*{\argmin}{arg\,min}

\title{\name: Efficient and Probably Approximately Correct (PAC) Identification of Top-k Features}

\author{%
  Sanjay Kariyappa\\
  J.P. Morgan AI Research\\
  Palo Alto, CA, USA \\
  \texttt{sanjay.kariyappa@jpmchase.com} \\
  \And
  Leonidas Tsepenekas\\
  J.P. Morgan AI Research\\
  New York City, NY, USA \\
  \texttt{leonidas.tsepenekas@jpmchase.com} \\
  \And
  Freddy Lécué\\
  J.P. Morgan AI Research\\
  New York City, NY, USA \\
  \texttt{freddy.lecue@jpmchase.com} \\
  \And
  Daniele Magazzeni\\
  J.P. Morgan AI Research\\
  London, UK \\
  \texttt{daniele.magazzeni@jpmorgan.com} \\
}

\newcommand{\ignore}[1]{}
\usepackage{biblatex}
\begin{document}

\maketitle

\begin{abstract}
The SHAP framework provides a principled method to explain the predictions of a model by computing feature importance. Motivated by applications in finance, we introduce the \emph{Top-k Identification Problem (TkIP)}, where the objective is to identify the $k$ features with the highest SHAP values. While any method to compute SHAP values with uncertainty estimates (such as KernelSHAP and SamplingSHAP) can be trivially adapted to solve TkIP, doing so is highly sample inefficient. The goal of our work is to improve the sample efficiency of existing methods in the context of solving TkIP. Our key insight is that  TkIP can be framed as an \emph{Explore-m problem}~\cite{explorem}--a well-studied problem related to multi-armed bandits (MAB). This connection enables us to improve sample efficiency by leveraging two techniques from the MAB literature: (1) a better stopping-condition (to stop sampling) that identifies when PAC (\emph{Probably Approximately Correct}) guarantees have been met and (2) a greedy sampling scheme that judiciously allocates samples between different features. By adopting these methods we develop \emph{KernelSHAP@k} and \emph{SamplingSHAP@k} to efficiently solve TkIP, offering an average improvement of $5\times$ in sample-efficiency and runtime across most common credit related datasets.
\end{abstract}
\section{Introduction}
The ability to explain the predictions of ML models is of critical importance in highly regulated industries, where laws provide \emph{right to explanations} for people who are adversely impacted by algorithmic decision making. 
Specifically in finance, regulations like Fair Credit Reporting Act~\cite{fcra} and Equal Credit Opportunity Act~\cite{ecoa} require a rejected loan/credit application (i.e. adverse action) to be explained to the borrower, by providing reasons for why the application was rejected (e.g., low credit score, high debit-to-income ratio, recent delinquencies, etc.).  
Owing to its principled formulation, the SHAP framework~\cite{shap} is the defacto choice for explaining model predictions in credit-risk assesment models~\cite{shap_finance}. While exact computation of SHAP values is computationally intractable,
sampling-based techniques like KernelSHAP~\cite{shap} and SamplingSHAP provide a practical alternative to compute approximate SHAP values. Additionally, recent works have developed methods to quantify the approximation error of such sampling-based techniques, by providing confidence intervals (CIs) for the estimated SHAP values~\cite{covert2020improving}.

In this paper, we introduce the \emph{Top-k Identification Problem (TkIP)}, where the objective is to identify the $k$ most important features, i.e., those with the $k$ highest SHAP values (referred to as the \emph{Top-k features}). TkIP is motivated by an important real-world use-case of processing credit/loan applications, where the lender is  required to provide the top features that contributed negatively to the model's prediction (i.e. explanations) in the event of a rejection; this is standard practice by credit/loan issuers in order to comply with the Equal Credit Opportunity Act~\cite{ecoa}. Existing methods like KernelSHAP and SamplingSHAP can be straightforwardly adapted to identify Top-k features with PAC guarantees, by evaluating enough samples to sufficiently reduce the the CIs of the SHAP estimates. However, doing so can be computationally expensive as it often requires a very large number of samples.

Motivated by this problem, our paper investigates methods to improve the sample efficiency of KernelSHAP and SamplingSHAP, specifically to solve TkIP.
Our key insight is that TkIP can be framed as an \emph{Explore-m problem}~\cite{explorem} -- a well-studied problem related to multi-arm bandits (MAB), where the goal is to identify a subset of arms with the highest expected payoffs. By leveraging this connection, we make the following key changes to the SHAP estimation algorithms based on ideas that have been developed in the MAB literature:
\begin{enumerate}[itemsep=1pt, leftmargin=*, topsep=0pt, label=(C\arabic*)]
\item\textbf{Overlap-based stopping condition~\cite{lucb}:} Sampling for KernelSHAP and SamplingSHAP is usually done until the CI widths of SHAP values associated with \emph{all the features} falls below a threshold. This naive stopping condition is unnecessarily conservative for solving TkIP; so instead, we use a stopping condition that is based on the overlap in CIs between different features (instead of the absolute CI width of each feature). This allows for early-stopping once a PAC solution for TkIP has been identified.
\item\textbf{Greedy sampling scheme~\cite{lucb}:} For SamplingSHAP, the default sampling scheme of allocating samples according to the variance of each feature is ill-suited for solving TkIP. Instead, we leverage a greedy sampling scheme that is designed to efficiently solve the \emph{Explore-m problem} by allocating a higher number of samples to features that are likely to change the Top-k subset. This enables a significant reduction in sample-costs compared to the variance-based sample allocation. Note that (C2) requires the ability to allocate samples to evaluate the SHAP values on a per-feature basis, so it cannot be applied to KernelSHAP.
\end{enumerate}

We use the above techniques to develop \emph{KernelSHAP@k} (\emph{KernelSHAP + C1}) and \emph{SamplingSHAP@k} (\emph{SamplingSHAP + C1 + C2}). We evaluate these methods with the most common credit related datasets and show that they offer significant improvements in sample efficiency and runtime, compared to their respective baselines. 
The rest of this paper is structured as follows:
\begin{itemize}[itemsep=1pt, topsep=0pt, leftmargin=20pt]
    \item In Section \ref{sec:related}, we provide background on sampling-based methods that can be used to estimate SHAP values and related work on variance reduction and uncertainty estimation.
    \item In Section \ref{sec:problem}, we formally define the \emph{Top-k Identification problem} and develop a \emph{naive stopping condition} that can be used with Kernel/Sampling SHAP to correctly identify Top-k features. Nonetheless, this condition is sample-inefficient.
    \item In Section \ref{sec:approach}, we develop \emph{KernelSHAP@k} and \emph{SamplingSHAP@k} to efficiently solve TkIP with PAC guarantees. The key insight here is framing TkIP as an \emph{Explore-m problem}.
    \item In Section \ref{sec:experiments}, we evaluate Kernel/Sampling-SHAP@k on a suite of credit related datasets and demonstrate significant improvements in sample-costs and runtime.
    \item We discuss limitations and future directions in Section~\ref{sec:limitations}, and conclude in Section \ref{sec:conclusion}.
\end{itemize}

\section{Background and related work}\label{sec:related}
The goal of our work is to modify existing algorithms to efficiently identify Top-k features with PAC guarantees. In this section, we provide background on the SHAP framework and discuss existing sampling-based techniques (SamplingSHAP and KernelSHAP) that estimate SHAP values. Additionally, we discuss related works that extend these method by reducing the variance of the estimates and quantify uncertainty in the form of confidence intervals. 

\subsection{SHAP}
SHAP (SHapley Additive exPlanations) is based on a game-theoretic concept called Shapley values~\cite{shapley}, which is a method to fairly distribute the payoffs of a cooperative game among the players. This is done by measuring the average marginal contribution of a single player computed across all possible coalitions of players. Such a formulation of assigning credit has been shown to uniquely satisfy a set of fairness axioms such as local accuracy, missingness and consistency~\cite{young1985monotonic}. SHAP applies this concept to explaining the predictions of the model by treating individual features as players and the output of the model as the payoff. By measuring the marginal contributions of features across different coalitions, SHAP assigns a score to each feature that reflects its contribution to the final prediction of the model. Given a set of features $D =\{1,2,..,d\}$, the SHAP value $\phi_i$ for the $i^{th}$ feature of an input $x$ with a model $f$ is computed by taking the weighted average of the change in predictions of $f$ when feature $i$ is added to a subset of features $S$ as shown in Eqn.\ref{eq:shap}.
\begin{equation}\label{eq:shap}
\phi_i(x,f) = \sum_{S \subseteq D \setminus \{i\}} \frac{|S|!(d-|S|-1)!}{d!} \left[f(x_{S \cup \{i\}}) - f(x_S)\right]
\end{equation}
Here $x_S$ is the feature vector restricted to $S$. To evaluate the model function with missing features in the above expression, we use the interventional SHAP formulation~\cite{chen20}, where missing feature values are set to a default baseline. Note that computing SHAP values exactly has a computational complexity of $\Theta(2^d)$. While there are efficient methods to compute exact SHAP values for specific models such as decision trees~\cite{treeshap1, treeshap2}, in general, the exponential complexity makes it computationally intractable to evaluate exact SHAP values when the number of features is large. To reduce computational costs, sampling-based approximation techniques have been proposed. We explain two such methods in the remainder of this section.

\subsection{SamplingSHAP}\label{sec:SS}
SamplingSHAP estimates SHAP values by only evaluating a subset of terms in Eqn.\ref{eq:shap} and then averaging over the resulting marginals. Štrumbelj et al.~\cite{sampling_shap} provide an efficient algorithm to perform Monte Carlo sampling according to the probability distribution induced by the weights in Eqn.\ref{eq:shap}. To quantify the uncertainty in the SHAP estimate based on the number of samples, Merrick et al.~\cite{explanation_game} proposed the use of Standard Error of Means (SEM) to derive confidence intervals through the \emph{Central Limit Theorem (CLT)}. Specifically, the Monte Carlo simulation is run $T_i$ times for each feature $i$, thus giving a set of SHAP estimates $\{\hat{\phi}_i^j\}_{j=1}^{T_i}$. Finally, the SHAP value for $i$ is set to be $\hat{\phi}_i = \sum^{T_i}_{j = 1}\hat{\phi}_i^j / T_i$. Eqn.\ref{eq:sampling_ci} shows how the $95\%$ CI for the $i^{th}$ feature (there's a $0.95$ probability of $\phi_i$ being in $CI_i$):
\begin{equation}\label{eq:sampling_ci}
CI_i = \left[\hat{\phi}_i \pm 1.96\frac{\sigma_i}{\sqrt{T_i}}\right].
\end{equation}
Here  $\sigma_i$ denotes the standard deviation of the set of SHAP estimates $\{\hat{\phi}_i^j\}_{j=1}^{T_i}$. Note that we can achieve any confidence that we want, by tweaking the parameter $1.96$ accordingly. 

Additionally, prior works have also tried to reduce the length of the CIs through variance reduction techniques. For instance, Mitchell et al.~\cite{sampling_paired} propose to evaluate negatively correlated pairs of samples in SamplingSHAP to reduce the variance $\sigma_i$ of SHAP estimates. Sampling techniques have also been used in the context of Game Theory for computing Shapley values~\cite{maleki2014bounding, ijcai2021p11, 10.1016/j.cor.2008.04.004}.

\subsection{KernelSHAP}\label{sec:KS}

KernelSHAP~\cite{shap} is another sampling-based method that views SHAP values as the solution to a weighted regression problem. Specifically, consider a linear model of the form $g(S) = \phi_0 + \sum\limits_{i\in S}\phi_i$, where $\phi_i$ denote the SHAP values. KernelSHAP proposes to estimate these values by solving the following optimization problem:
\begin{equation}\label{eq:kernel_shap}
\{\phi_i\} = \argmin\limits_{\phi_1,..\phi_d} \sum\limits_{S\subseteq D}w(S)\Big(f(S)-g(S)\Big)^2. 
\end{equation}
Here, $w(S)$ is a weighting function that is chosen in a way that makes solving Eqn.\ref{eq:kernel_shap} equivalent to finding SHAP values. Note that evaluating Eqn.\ref{eq:kernel_shap} requires evaluating an exponential number of terms in the summation, making the computation of exact SHAP values intractable. Fortunately, an approximation of Eqn.\ref{eq:kernel_shap} that evaluates only a small subset of terms is sufficient in practice to estimate SHAP values. Furthermore, a recent work~\cite{covert2020improving} has shown that the variance of SHAP values, computed by using KernelSHAP, can be used to derive confidence intervals, providing a means of detecting convergence in the SHAP estimates; this leads to CIs identical to those of Eqn.\ref{eq:sampling_ci}. Additionally, this work also uses paired-sampling (similar to ~\cite{sampling_paired}) with KernelSHAP to reduce computational costs, by reducing the variance of the SHAP estimates.

\section{Problem setting}\label{sec:problem}
In this section, we formally define the Top-k identification problem (TkIP), the goal of which is to identify the features with the highest SHAP values. To apply sampling based techniques to solve TkIP, we define an $(\epsilon, \delta)$-PAC solution for it, which allows for an $\epsilon$-approximate version of the solution with a low probability of failure ($\delta$). Finally, we describe a naive stopping condition that can be used with Kernel/Sampling-SHAP to derive a $(\epsilon, \delta)$-PAC solution. We demonstrate that this naive solution is sample-inefficient, motivating the need for our proposed solutions that improve sample-efficiency.

\subsection{Top-k identification problem}
Consider a model $f:\mathcal{I}\rightarrow\mathbb{R}$, which acts on a $d$-dimensional input $x\in\mathcal{I}$ to produce a prediction $p=f(x)$. For an input $x\in \mathcal{I}$, let $\{\phi_1, \phi_2, ..., \phi_d\}$ denote the set of SHAP values corresponding to the input features $D=\{1, 2,.., d\}$ respectively. To simplify notation, let us assume that the features are indexed such that:
\begin{equation}
    \phi_1 \geq \phi_2 \geq \phi_3.. \geq \phi_d.
\end{equation}
The goal of TkIP is to identify the $k$ features: \textsc{Topk}$ = \{1,2,..,k\}$ corresponding to the $k$ highest SHAP values $\mathcal{S}=\{\phi_1, \phi_2, .. \phi_k\}$. Note that the ordering of features in \textsc{Topk} does not matter. Solving TkIP exactly requires us to precisely evaluate all the SHAP values, which is computationally intractable. Instead, we define \emph{$\epsilon-$approximate} and \emph{($\epsilon, \delta$)-PAC} solutions for TkIP that are more useful in the context of sampling-based PAC methods.
\begin{itemize}[itemsep=1pt, topsep=0pt, leftmargin=10pt]
    \item $\boldsymbol{\epsilon}$-\textbf{approximate solution:} For a given accuracy parameter $\epsilon \in (0,1)$, consider a subset of features $D^*\subset D$ such that $|D^*| = k$. $D^*$ is an $\epsilon$-approximate solution to TkIP if it satisfies the following: 
    \begin{equation}
    \phi_i \geq \phi_k-\epsilon, \forall i \in D^*.   
    \end{equation}
    \item $\boldsymbol{(\epsilon, \delta)}$-\textbf{PAC solution:} For given accuracy and confidence parameters $\epsilon, \delta \in (0,1)$, $D^*$ is said to be an $(\epsilon, \delta)$ solution for TkIP if it is an $\epsilon$-approximate solution with a probability at least $1-\delta$:
    \begin{equation}
    \Pr[\phi_i \geq \phi_k-\epsilon, \forall i \in D^*] \geq 1-\delta.
    \end{equation}
    In other words, here we allow for randomized algorithms that should compute $D^*$ with controllable (low) probability of failure.
\end{itemize}
This relaxed notion of the solution allows for a feature $i$ to be returned as part of the solution even if $i \notin$ \textsc{Topk}, as long as the corresponding SHAP value $\phi_i$ is $\epsilon$-close to $\phi_k$ (i.e. the $k^{th}$ SHAP value).

\subsection{PAC solution for TkIP with naive stopping condition}
In both KernelSHAP and SamplingSHAP, we can use the CLT-based approaches mentioned in Sections~\ref{sec:SS}, \ref{sec:KS} to obtain confidence intervals of the following form. Let $\phi_i$ be the true SHAP value for feature $i$, and let $\hat{\phi}_i$ be our approximation for it. Then, if we repeat the corresponding algorithm $T_i$ times, with probability at least $1-\frac{\delta}{d}$ we have:
\begin{align}
    |CI_i| = 2 \cdot Z(\delta / d) \frac{\sigma_i}{\sqrt{T_i}}\label{CI}
\end{align}
In the above, $Z(\delta / d)$ is the critical value from the standard normal distribution for the desired level of confidence; note that this value is a small constant. It is clear from Eqn.~\ref{CI}, that the larger $T_i$ is, the closer our approximation is to the true value. One way to identify the \textsc{Topk} features is by running the SHAP estimation algorithm (i.e. adding more samples) until the CIs for all the features are small enough to meet the following stopping condition: 
\begin{align}\label{eq:naive_stopping_condition}
|CI_i| = 2 \cdot Z(\delta / d) \frac{\sigma_i}{\sqrt{T_i}} \leq \epsilon, \forall i\in D. 
\end{align}
We call this the \emph{naive stopping condition}, and in Theorem~\ref{thm:naive} we show that it indeed leads to an $(\epsilon, \delta)$-PAC solution for TkIP. Thus, Kernel/Sampling-SHAP can be straightforwardly adapted to solve TkIP by using enough samples to meet this stopping condition. In the following subsection, we will explain why this naive approach is sample-inefficient with the aid of an example, motivating the need for a better stopping condition and sampling technique.
\begin{figure*}[t]
  \centering
  \includegraphics[width=\textwidth]{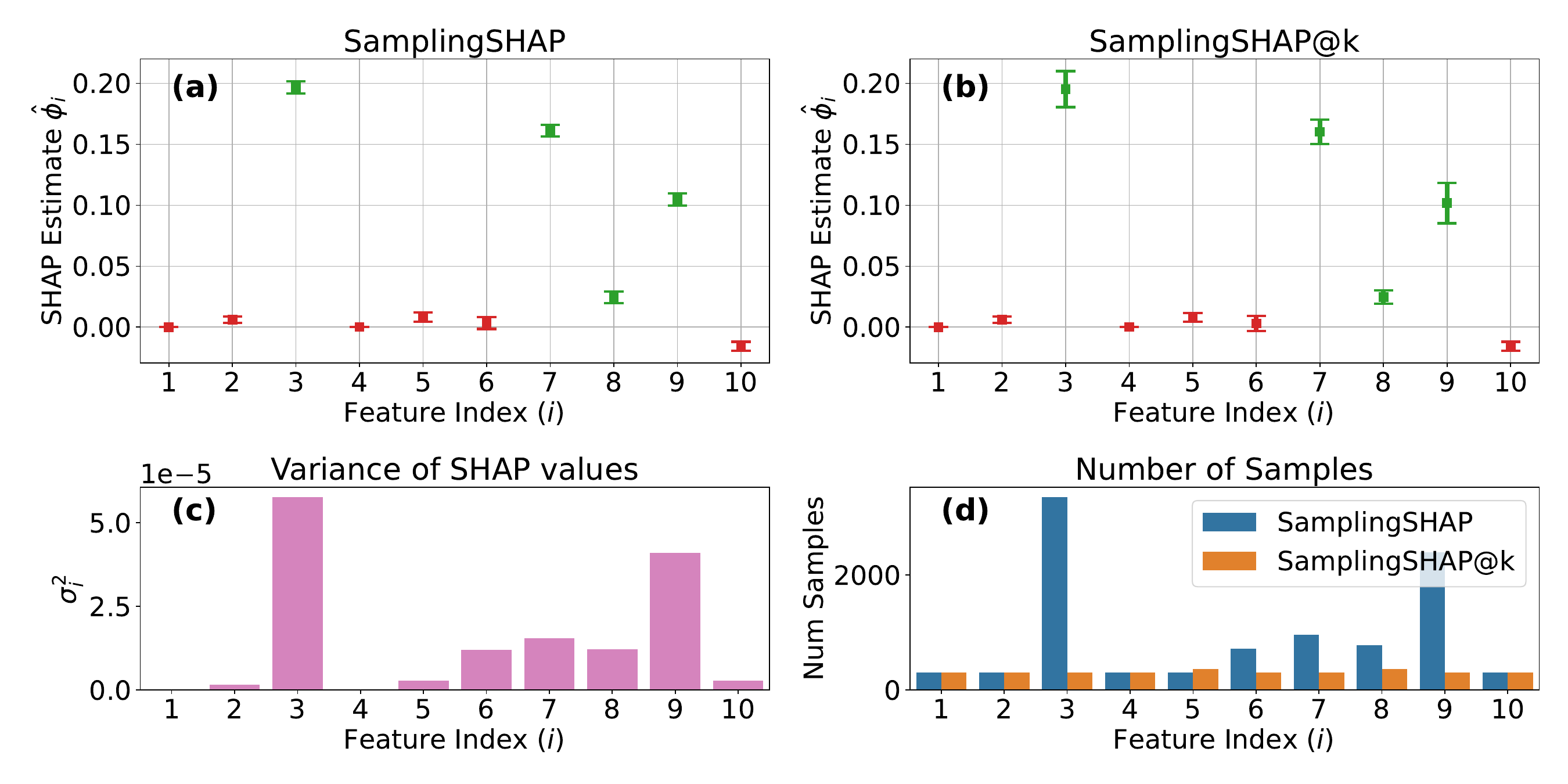}
  \vspace{-0.2in}
  \caption{Comparing SamplingSHAP and SamplingSHAP@k for solving TkIP. (a) CIs of SHAP estimates with SamplingSHAP (b) CIs of SHAP estimates with SamplingSHAP@k (c) Variance of SHAP estimates (d) Costs of SamplingSHAP and SamplingSHAP@k. The cost of SamplingSHAP scales with the variance of the SHAP estimates, resulting in high cost. SamplingSHAP improves sample efficiency by using an improved stopping condition and sampling scheme.}
  \label{fig:naive_vs_lucb}
\end{figure*}

\begin{theorem}\label{thm:naive}
Let $\mathcal{S}=\{\hat{\phi_1}, \hat{\phi_2},.., \hat{\phi}_d\}$ denote the SHAP estimates of input features $D=\{1,2,..,d\}$, such that the $CI_i$s defined using a confidence of $\frac{\delta}{d}$ satisfy $|CI_i| \leq \epsilon, \forall i\in D$. Then, $D^*=\argmax_k(\mathcal{S})$ is an $(\epsilon, \delta)$-PAC solution for TkIP; the solution consists of the $k$ features with the largest $\hat{\phi}_i$.
\end{theorem}
\begin{proof}
We show that when $\phi_i \in CI_i$ for every $i$, the solution is $\epsilon$-approximate. Using a union bound over all features we have:
\begin{align*}
    \Pr[\phi_i \in CI_i, \forall i] = 1 - \Pr[\exists i: \phi_i \notin CI_i] \geq 1 - \sum^d_{i = 1}\frac{\delta}{d} = 1 - \delta
\end{align*}
For the inequality above we used the definition of $CI_i$, which states that $\Pr[\phi_i \notin CI_i] \leq \delta / d$.

Clearly, if we prove that $\phi_i \in CI_i, \forall i$ implies an $\epsilon$-approximate solution we are done. Therefore, for the sake of contradiction, assume that the resulting solution is not $\epsilon$-approximate. This means that there exists feature $\tilde{i}$ with $\phi_{\tilde{i}} < \phi_k - \epsilon$, which still made it in our top-$k$ solution. By definition of \textsc{Topk} and $CI_i$, we have that for all $i \in \textsc{Topk}$ $\hat{\phi}_i \geq \phi_k - \frac{\epsilon}{2}$. By definition of $CI_{\tilde{i}}$, we have $\hat{\phi}_{\tilde{i}} \leq \phi_{\tilde{i}} + \frac{\epsilon}{2}$. Combining this with $\phi_{\tilde{i}} < \phi_k - \epsilon$, gives $\hat{\phi}_{\tilde{i}} < \phi_k - \frac{\epsilon}{2}$. Hence, $\tilde{i}$ could never be chosen instead of any $i \in \textsc{Topk}$ in the returned solution.
\end{proof}

\subsection{Understanding the inefficiencies of the naive stopping condition}

The Naive stopping condition requires the CIs of all the features to be of width at most $\epsilon$. For a feature $i$, the number of samples $N_i$ necessary to achieve this is proportional to the variance of the feature's SHAP estimate ($N_i\propto \sigma^2_i$), resulting in high-variance features incurring a higher sample-cost. To illustrate, we apply SamplingSHAP to explain the prediction of an MLP model on a single example from the UCI Credit dataset. To identify the Top-k features (with $k=4$), we obtain CIs by runing SamplingSHAP multiple times for each feature, until the stopping condition in Eqn.~\ref{eq:naive_stopping_condition} is met. We visualize the CIs of the SHAP estimates of the individual features in Fig.\ref{fig:naive_vs_lucb}a, where the Top-4 features are marked as green. To understand the cost of this stopping condition, we plot the number of function evaluations consumed by the algorithm in Fig.\ref{fig:naive_vs_lucb}d and the variance of the SHAP estimate for each feature in Fig.\ref{fig:naive_vs_lucb}c. As expected, we find that the cost is proportional to the variance of the per-feature SHAP estimate, resulting in a high sample-cost for high-variance features.

A key drawback of the naive sampling scheme is that it requires $|CI_i| \leq \epsilon$ for all features, regardless of the uncertainty that the feature belongs in \textsc{Topk}. This results in a lot of wasted samples. For instance, in the example in Fig.\ref{fig:naive_vs_lucb}, $\hat{\phi}_3$ (SHAP estimate for $feature\text{-}3$) is much higher compared to the other features, allowing us to conclude with high confidence that $3\in \textsc{Topk}$ early on in the sampling process and avoid sampling $feature\text{-}3$ further. However the naive sampling scheme lacks such adaptivity and forces this high-variance feature to continue sampling until $|CI_3| \leq \epsilon$, thus leading to a lot of wasted samples and contributing significantly to the sample cost of SamplingSHAP. In the next section we develop $SamplingSHAP@k$ and $KernelSHAP@k$ to avoid such wasted samples by using a modified stopping condition and sampling scheme.

\section{SHAP@k: Framing TkIP as an Explore-m problem}\label{sec:approach}
The key insight of our work is that TkIP can be framed as an \emph{Explore-m} problem--a well-studied problem in multi-armed bandits (MAB), where the goal is to identify the arms with the highest expected payoffs in a sample-efficient way \cite{lucb, MAB-2, MAB-3, MAB-4}. Formally, given N arms, each with some unknown distribution of payoffs, the objective is to identify (with PAC guarantees) the subset of $m$ arms with the highest expected payoff. Note that TkIP has a 1-1 correspondence with the Explore-m problem. The arms in MAB are equivalent to the features in the context of SHAP, and the reward obtained by pulling an arm is equivalent to the SHAP estimate of a specific feature obtained through a single sample. The goal is to identify the subset of $m\ arms$/$k\ features$ with the highest expected rewards/SHAP values. This connection allows us to leverage methods from the MAB literature to efficiently solve TkIP. Hence, we propose changes to the earlier sampling scheme and stopping condition, to develop sample efficient variants of Kernel and Sampling SHAP.

\subsection{Overlap-based stopping condition (C1)}
Inspired by Kalyanakrishnan et al.~\cite{lucb}, we use the stopping condition in Theorem~\ref{thm:improved_stopping_condition} that considers the overlap in CIs between the SHAP estimates of different features. By only considering the overlap between the CIs, the improved stopping condition avoids the need to reduce all the CIs widths to below $\epsilon$ as shown in Fig.~\ref{fig:naive_vs_lucb}b. Through experimental evaluations, we show that compared to the naive-stopping condition, this results in a significant reduction in the number of samples necessary to identify the \textsc{Topk} features (Fig.~\ref{fig:naive_vs_lucb}d).

We now introduce some notation. Let $T_i$ the number of SHAP estimates that we have collected so far for feature $i$. For the desired confidence $\delta$, we define a $\delta / d$ confidence interval $CI_i = [\alpha_i, \beta_i]$ as before, where $\hat{\phi}_i$ the current SHAP estimate, $\alpha_i = \hat{\phi}_i - Z(\delta / d)\frac{\sigma_i}{\sqrt{T_i}}$ and $\beta_i = \hat{\phi}_i + Z(\delta / d)\frac{\sigma_i}{\sqrt{T_i}}$. 
 
\begin{theorem}\label{thm:improved_stopping_condition}
Let $High$ denote the set of $k$ features with the highest SHAP estimates $\hat{\phi_i}$ and $Low$ denote the remaining set of $d-k$ features. Let $h$ be the feature in $High$ with the lowest lower confidence bound i.e. $h = \argmin\limits_{i\in High}\{\alpha_i\}$, and let $\ell$ be the feature in $Low$ with the highest higher confidence bound i.e. $\ell=\argmax\limits_{i\in Low}\{\beta_i\}$. Then, $High$ is a $(\epsilon, \delta)$-PAC solution for TkIP if the following condition is satisfied:
$
    \beta_\ell - \alpha_h \leq \epsilon.
$
\end{theorem}
\begin{proof}
This proof is identical to Theorem 1 from \cite{lucb} with one minor difference. The authors in \cite{lucb} use Hoeffding's inequality prior to taking a union bound to show that the failure probability is at most $\delta$. Here, we do not need the application of Hoeffding's inequality, since we alreay have the CLT guarantees for the CIs.
\end{proof}
\begin{algorithm}
    \caption{Greedy Sampling Algorithm (with Overlap-Based Stopping Condition)}
    \label{alg:lucb}
    \begin{algorithmic}[1]
        \Procedure{Greedy}{$f$, $x$, $ShapEstAlg$, $T_{min}$, $\epsilon$, $\delta$}
            \State// Collect a minimum number of shap estimates for all features
            \For{$i \gets 1$ to $d$}
                \For{$j \gets 1$ to $T_{min}$}
                    \State$\hat{\phi}^j_i = ShapEstAlg(f,x, i)$
                \EndFor
                $S_i = \{\hat{\phi}^j_i\}$
                \State$\hat{\phi}_i \gets \sum_j \hat{\phi}^j_i / T_{min}$, $\alpha_i \gets \hat{\phi}_i - Z(\delta / d)\frac{\sigma_i}{\sqrt{T_{min}}}$ and $\beta_i \gets \hat{\phi}_i + Z(\delta / d)\frac{\sigma_i}{\sqrt{T_{min}}}$
            \EndFor
            \While{ True}
                \State// Identify $h,\ell$
                \State$High, Low \gets \argmax_k(\{\hat{\phi_i}\}), \argmin_{d-k}(\{\hat{\phi_i}\})$
                \State$h, \ell \gets \argmin\limits_{i\in High}( \alpha_i), \argmax\limits_{i\in Low}(\beta_i)$ 
                \State// collect additional SHAP estimates for $h,\ell$
                \State$S_h \gets S_h\cup ShapEstAlg(f,x,h)$ 
                \State$S_{\ell} \gets S_{\ell}\cup ShapEstAlg(f,x,\ell)$ 
                \State// Update mean and CI widths of $h,\ell$
                \State For $i \in \{\ell, h\}$: $\hat{\phi}_i \gets \sum_j \hat{\phi}^j_i / |S_i|$, $\alpha_i \gets \hat{\phi}_i - Z(\delta / d)\frac{\sigma_i}{\sqrt{|S_i|}}$ and $\beta_i \gets \hat{\phi}_i + Z(\delta / d)\frac{\sigma_i}{\sqrt{|S_i|}}$ 
                \State Break if $\beta_{\ell} - \alpha_h \leq \epsilon$
            \EndWhile
            \State \textbf{return} $High$
        \EndProcedure
    \end{algorithmic}
\end{algorithm}
\vspace{-0.2in}
\subsection{Greedy sampling scheme (C2)}
The default variance-based sampling scheme used by Sampling SHAP minimizes the CIs for all features. Such sampling schemes are inefficent for the stopping condition in Theorem~\ref{thm:improved_stopping_condition}, which only depends on two features ($h$ and $\ell$) at any given point in the sampling process. To improve the sample efficiency, we consider a greedy sampling strategy~\cite{lucb} as described in Algorithm~\ref{alg:lucb}. The algorithm starts by using any feature-wise SHAP estimation algorithm (e.g., SamplingSHAP) to find an initial set of SHAP estimates $\{\hat{\phi}_i^j\}$ for each input feature $i$; \emph{a feature-wise SHAP estimator computes the SHAP values independently for each feature}. The mean SHAP estimates are used to categorize the features into the two groups $High, Low$. Then, the algorithm identifies $h$ and $\ell$ as defined in Threorem~\ref{thm:improved_stopping_condition}, and evaluates additional SHAP estimates for these two features. These steps are repeated until the stopping condition is met. At this point, $High$ will be a valid $(\epsilon, \delta)$-PAC solution for $TkIP$. This scheme improves sample efficiency by allocating more samples to $(h,\ell)$, which are exactly the features that can potentially affect what is inside \textsc{Topk}. To see why this algorithm terminates, notice that in each iteration exactly 2 CIs shrink. Therefore, in the worst case, there will come a point where all CIs will be of length at most $\epsilon$, and thus the stopping condition will trivially be true.

\subsection{KernelSHAP@k and SamplingSHAP@k}
We apply the above changes to existing algorithms to propose KernelSHAP@k (KernelSHAP + C1) and SamplingSHAP@k (SamplingSHAP + C1 + C2). In both cases, we incrementally add SHAP estimates $\hat{\phi}^j_i$ until the stopping condition (C1) is met and the \textsc{Topk} features are identified. Additionally, for SamplingSHAP@k, we use the more efficient greedy sampling scheme (C2) that allocates samples only to features that influence the stopping condition. Note that the greedy sampling scheme (C2) requires the ability to compute the SHAP values of features individually. Thus, we cannot apply C2 to KernelSHAP as it estimates the SHAP values of all features together. In contrast, SamplingSHAP estimates SHAP values per-feature, which makes it compatible with C2.

\begin{table}[b]
\vspace{-0.1in}
\caption{Description of the datasets used in our experiments.}
\centering
\scalebox{0.8}{
\begin{tabular}{l|l|l|l|l}
\toprule
\textbf{Datasets} & \textbf{\# Feats}& \textbf{\# Train} & \textbf{\#Test (Neg. pred)} & \textbf{Binary Classification Task} \\
\midrule
German Credit Risk~\cite{german}      & 61  & 800   &  56  & Loan repayment                       \\
Give Me Some Credit~\cite{gmsc}     & 11  & 96215 &  100 & Default in next 2 years                     \\
HELOC~\cite{heloc}                   & 23  & 8367  &  100 & Missing loan payments                       \\
UCI Credit Approval~\cite{uci}     & 15  & 522   &  66  & Credit card aproval/rejection                       \\
\bottomrule
\end{tabular}
}
\label{table:datasets}
\end{table} 
\section{Experiments} \label{sec:experiments}
To quantify the improvements in sample efficiency of our proposed methods, we compare the sample cost (i.e. number of function evaluations) of Kernel/SamplingSHAP@k with that of Kernel/SamplingSHAP (with naive stopping condition) using various credit-realted datasets. We present the experimental setup, followed by the results comparing sample costs and sensitivity studies that quantify how these costs change with the accuracy parameter $\epsilon$.
\subsection{Experimental setup}
Table\ref{table:datasets} lists the datasets used in our experiments, along with a brief description of the prediction task, number of features, and train/test split. In each case, we train a 5-layer MLP model on the binary classification task using the training set for 100 epochs, and use this model to make predictions on the test set. For the negatively classified examples in the test set (indicating a high likelihood of the credit application being rejected), we use different methods to compute the $Top\text{-}4$ features that contributed the most to the negative prediction in terms of their SHAP values\footnote{Our methodology of only evaluating explanations for negatively outcomes is motivated by regulations that require explainations to be provided in case of adverse actions (e.g., credit application being rejected).}. We use interventional SHAP for our experiments and use a positively classified example from the training set as our baseline. We compare the sample-efficiency of various methods in terms of the number of function ($f$) evaluations and runtime required to identify the $Top\text{-}4$ features with PAC guarantees\footnote{Runtime measured on a machine with 32-core AMD CPU and 128GB of memory. Code to reproduce results is included in the supplementary material.}.
\begin{table}[tb]
\caption{Comparing the sample cost and runtime required for finding $(\epsilon=0.005, \delta=10^{-6})$-PAC solution for TkIP ($k=4$) using different methods. Kernel/SamplingSHAP@k are faster and require significantly fewer function evaluations compared to Kernel/SamplingSHAP.}
\centering
\scalebox{0.8}{
\begin{tabular}{lllll}
\toprule
\multicolumn{1}{c}{\multirow{2}{*}{\textbf{Methods}}} & \textbf{German Credit} & \textbf{Give Me Some Credit} & \textbf{HELOC} & \textbf{UCI Credit} \\
\cmidrule{2-5}
\multicolumn{1}{c}{}                                  & \multicolumn{4}{c}{\textbf{Sample Cost (Num. of $f()$ evals.)}} \\
\cmidrule{2-5}
KernelSHAP & 221114 & 6883 & 59230 & 34336  \\
KernelSHAP@k & \textbf{27632}$(\downarrow8.0\times)$ & 5593$(\downarrow1.2\times)$ & 12986$(\downarrow4.6\times)$ & \textbf{7275}$(\downarrow4.7\times)$ \\
\cmidrule{2-5}
SamplingSHAP & 397607 & 10501 & 95641 & 53460 \\
SamplingSHAP@k & 27982$(\downarrow14.2\times)$ & \textbf{3184}$(\downarrow3.3\times)$ & \textbf{9221}$(\downarrow10.4\times)$ & 14485$(\downarrow3.7\times)$ \\
\midrule
\multicolumn{1}{c}{}                                  & \multicolumn{4}{c}{\textbf{Runtime (Seconds)}} \\
\cmidrule{2-5}
KernelSHAP & 15.95 & 0.04 & 0.5 & 3.34 \\
KernelSHAP@k & 2.28$(\downarrow7.0\times)$ & 0.03$(\downarrow1.2\times)$ & 0.11$(\downarrow4.6\times)$ & 0.73$(\downarrow4.6\times)$ \\
\cmidrule{2-5}
SamplingSHAP & 2.2 & 0.05 & 0.5 & 0.38 \\
SamplingSHAP@k & \textbf{0.15}$(\downarrow14.7\times)$ & \textbf{0.02}$(\downarrow3.3\times)$ & \textbf{0.05}$(\downarrow10.6\times)$ & \textbf{0.07}$(\downarrow5.1\times)$ \\
\bottomrule
\end{tabular}
}
\label{table:results}
\end{table}  

\subsection{Results}
Table\ref{table:results} compares the average sample cost (i.e. number of function evaluations) and average runtime required by different methods to identify $Top\text{-}4$ features with a $(\epsilon=0.005, \delta=10^{-6})$-PAC guarantee across different datasets. Our evaluations show that Kernel/SamplingSHAP@k significantly outperform their baseline counterparts Kernel/SamplingSHAP, offering between $1.2\times-14.2\times$ improvement in sample efficiency and between $1.2\times -14.7\times$ improvement in runtime. Between SamplingSHAP@k and KernelSHAP@k, we find that the method with the better sample-cost depends on the dataset in question. However, SamplingSHAP@k has a consistently lower runtime compared to kernelSHAP@k, even in cases when it has a higher sample cost. For instance, for the UCI credit dataset, we find that SamplingSHAP@k has roughly twice the sample cost compared to KernelSHAP@k, but it is $10\times$ faster in terms of runtime. The reason for this is that each KernelSHAP estimate is more expensive to compute as it requires solving a weighted regression problem using the outputs of the model. In contrast, SamplingSHAP works by just computing a simple average on the outputs of the model, which requires much less compute, resulting in a faster runtime.

\subsection{Sensitivity studies}
To understand how the accuracy parameter $\epsilon$ influences the sample-efficiency of various methods, we perform sensitivity studies by varying $\epsilon$ between $[0.005, 0.01]$. For different values of $\epsilon$, we plot the sample-cost (i.e. number of function evaluations) and runtime of different methods across the four datasets considered in our experiments. Note that a lower value of $\epsilon$ implies a lower margin of error in identifying the $Top-4$ features and requires estimating SHAP values with greater precision (narrower CIs). As $\epsilon$ is reduced from $0.01$ to $0.005$, we find that the sample-costs and runtimes of all methods increase. Notably, the rate of this increase is much higher for Sampling/KernelSHAP, compared to Sampling/KernelSHAP@k. This is because the naive stopping condition used by Sampling/KernelSHAP requires the CI widths of the SHAP estimates of all features to be lower than $\epsilon$, which drives up the samples required. In contrast, the stopping condition used by Sampling/KernelSHAP@k, allows for the CI widths of the features that don't influence the stopping condition to be much higher than $\epsilon$ and thus requires fewer samples.
\begin{figure*}[t]
  \centering\label{fig:sensitivity}
  \includegraphics[width=\textwidth]{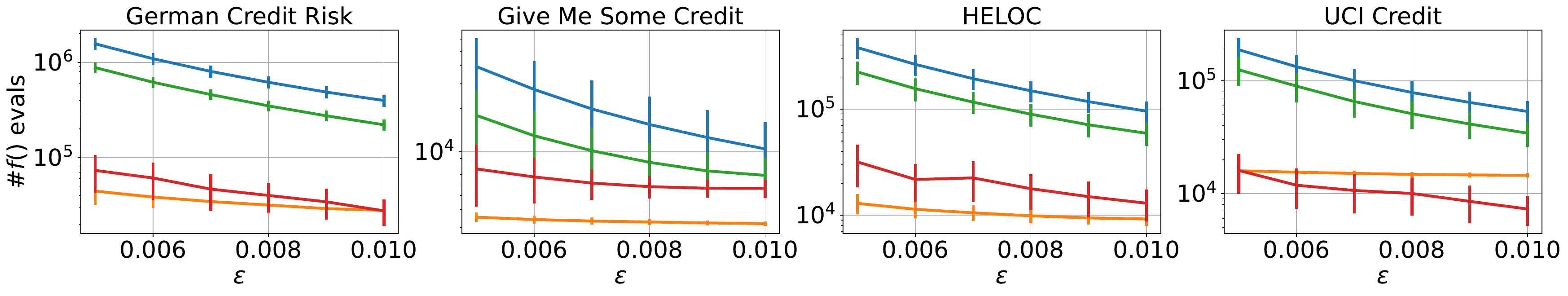}
  \includegraphics[width=\textwidth]{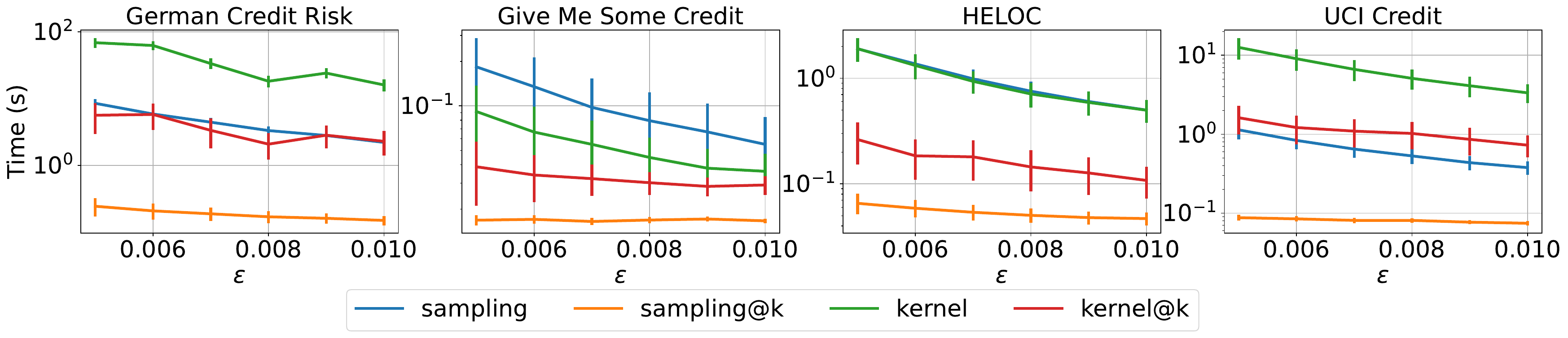}
  \caption{Sensitivity studies to understand sample-costs and runtimes of different methods for different values of the accuracy parameter $\epsilon$. While a lower value of $\epsilon$ results in an increase in sample-cost and runtime for all methods, this increase is much more significant for Kernel/SamplingSHAP compared to Kernel/SamplingSHAP@k. }
\end{figure*}
\section{Limitations and future work}\label{sec:limitations}
We discuss the limitations of our work and future directions of research in this section.

\textbf{Feature dependence:}  Since our work builds on the SHAP framework, it shares the limitations of SHAP. Importantly, SHAP assumes that the features of the input are not correlated. This assumption is typically not true in most practical settings. To address this issue, methods like GroupSHAP~\cite{lin2022model} have been developed, which groups features that are highly correlated and assigns attributions to groups of features instead of individual features. We leave the evaluation of our methods to the GroupSHAP setting as part of future work.

\textbf{Ordering of Top-k features:} Our proposed methods only solves the problem of identifying the Topk features. The features returned by our methods may not be in the right order. Thus, our methods may not be well suited for applications where the order of reporting the top-k features is important. One way in which our methods can be modified to such setting is by the repeated application of Kernel/SamplingSHAP@k by setting different values of $k$ ranging from $1,2,..k$. This would result in Topk features being identified in the right rank order. We leave the evaluation of this method as part of future studies.
\section{Conclusion}\label{sec:conclusion}
This paper studies the \emph{Top-k Identification problem (TkIP)}-- a novel problem setting, where the goal is to identify the $k$ features with the highest SHAP values. TkIP is motivated by applications in finance, where explanations for adverse actions are typically provided by listing the top-k features that led to a negative outcome. We find that while existing black-box techniques like KernelSHAP and SamplingSHAP can be trivially adapted to solve TkIP, doing so is highly sample inefficient. To address this issue, we develop sample efficient variants of these methods that are designed specifically for solving TkIP. Our key insight is that TkIP can be viewed as an \emph{Explore-m} problem -- a well-studied problem related to multi-armed bandits (MAB). This connection allows us to improve sample efficiency by using (1) an overlap-based stopping-condition and (2) a greedy sampling scheme that efficiently allocates samples between different features. We leverage these techniques to develop Kernel/SamplingSHAP@k, which can efficiently identify the Topk features with ($\epsilon, \delta$)-PAC guarantees . Our experiments on several credit-related datasets show that Kernel/SamplingSHAP@k significantly outperform their corresponding baselines: Kernel/SamplingSHAP , offering an average improvement of $5\times$ in sample-efficiency and runtime. We also characterize the sample-costs and runtime of our proposed methods across different levels of accuracy ($\epsilon$). Our paper provides efficient solutions to a previously unstudied problem that has important practical applications in finance.
\section{Acknowledgements}
This paper was prepared for informational purposes by the Artificial Intelligence Research group of JPMorgan Chase \& Co and its affiliates (“J.P. Morgan”) and is not a product of the Research Department of J.P. Morgan.  J.P. Morgan makes no representation and warranty whatsoever and disclaims all liability, for the completeness, accuracy or reliability of the information contained herein.  This document is not intended as investment research or investment advice, or a recommendation, offer or solicitation for the purchase or sale of any security, financial instrument, financial product or service, or to be used in any way for evaluating the merits of participating in any transaction, and shall not constitute a solicitation under any jurisdiction or to any person, if such solicitation under such jurisdiction or to such person would be unlawful. 

\printbibliography  
\end{document}